\documentclass[12]{article}
\usepackage{authblk}

\usepackage{amssymb, amsmath, amsthm}

\usepackage{color, enumerate}
\usepackage{epsfig, graphicx, subfigure}
\usepackage{algorithm, algorithmic}
\usepackage{array}

\usepackage{comment}

\theoremstyle{plain} 
  \newtheorem{thm}{Theorem}[section]
  \newtheorem{prop}[thm]{Proposition}
  \newtheorem{lem}[thm]{Lemma}

\theoremstyle{definition} 

\theoremstyle{remark} 

\newcommand{\bfx}{\mathbf{x}}
\newcommand{\bfy}{\mathbf{y}}

\newcommand{\mbR}{\mathbb{R}}

\newcommand{\mcN}{\mathcal{N}}
\newcommand{\mcR}{\mathcal{R}}
\newcommand{\mcM}{\mathcal{M}}
\newcommand{\mcV}{\mathcal{V}}
\newcommand{\mcE}{\mathcal{E}}
\newcommand{\mcG}{\mathcal{G}}

\newcommand{\tild}{\tilde{d}}

\newcommand{\bfv}{\mathbf{v}}

\newcommand{\bfe}{\mathbf{e}}

\newcommand{\bfzr}{\mathbf{0}}

\definecolor{blue}{rgb}{0,0,0.9}

\definecolor{red}{rgb}{0.9,0,0}

\title{Formation-Controlled Dimensionality Reduction}

\author[1]{Taeuk Jeong}
\affil[1]{Department of Computational Science and Engineering\\ \newline
Yonsei University\\ ({\tt iamlogin@yonsei.ac.kr})}

\author[2]{Yoon Mo Jung}
\affil[2]{Department of Mathematics\\ \newline
Sungkyunkwan University\\ ({\tt ymjung@skku.edu})}

\author[3]{Euntack Lee}
\affil[3]{Department of Mathematics\\ \newline
Sungkyunkwan University\\ ({\tt etlee@skku.edu})}

\begin{document}
\maketitle

\begin{abstract}
    Dimensionality reduction represents the process of generating a low dimensional representation of high dimensional data. Motivated by the formation control of mobile agents, we propose a nonlinear dynamical system for dimensionality reduction. The system consists of two parts; the control of neighbor points, addressing local structures, and the control of remote points, accounting for global structures.We also include a brief mathematical analysis of the model and its numerical procedure. Numerical experiments are performed on both synthetic and real datasets and comparisons with existing models demonstrate the soundness and effectiveness of the proposed model.

\smallskip
\noindent \textbf{Keywords.} Dimensionality reduction, Manifold Learning, Formation Control, Dynamical System

\end{abstract}

\tableofcontents

\section{Introduction}
Dimensionality reduction represents the process of extracting low dimensional structure from high dimensional data. High dimensional data include multimedia databases, gene expression microarrays, and financial time series, for example. In order to deal with such real-world data properly, it is better to reduce its dimensionality to avoid undesired properties of high dimensions such as the curse of dimensionality~\cite{vandermaaten2009, SaulWeinberger06}. As a result, classification, visualization, and compression of data can be expedited, for example~\cite{vandermaaten2009}.

In many problems, it is presumed that the dimensionality of the measured data is only artificially high; the measured data are high-dimensional but data nearly have a lower-dimensional structure, since they are multiple, indirect measurements of an underlying factors, which typically cannot be directly calibrated~\cite{Ghodsi2006}. Dimensionality reduction is the transformation of such data into a meaningful representation of reduced dimensionality. Ideally, the reduced representation should have a dimensionality that corresponds to the intrinsic dimensionality of the data~\cite{vandermaaten2009, Ghodsi2006}.

Traditionally, linear methods were employed for dimensionality reduction, such as principal components analysis, multidimensional scaling~\cite{Cunningham15}. To overcome the limitation of these methods, diverse nonlinear methods have been introduced in the last two decades, Laplacian Eigenmap, Isomap, Locally Linear Embedding, kernel PCA, Non-negative Matrix Factorization, just to name a few. For a comprehensive overview and classification of dimensionality reduction techniques, we refer the interested readers to the review papers~\cite{vandermaaten2009, SaulWeinberger06, Ghodsi2006, Sorzano2014, Cunningham15}.

In this paper, we present a different perspective-formation control~\cite{Ahn20, OhParkAhn15, Olfati-Saber06, Sun18}. Motivated by the formations of mobile agents under interagent distance control, we regard the dimensionality reduction process as interaction between many bodies, moving toward a desired formation by keeping local distances. This approach offers a fresh insight and vision into existing methods.

This paper is organized as follows. In Section 2, we review related models and provide motivation for the proposed model. Section 3 introduces the new model, employing a nonlinear dynamical system. Section 4 briefly studies mathematical properties of the model. In Section 5, we provide a computational scheme and numerical experiments. Finally, Section 6 concludes the paper.

\section{Related Models}
Let $\bfx_i$,  $i = 1, \ldots, n$ be $n$ \emph{high dimensional} data points in $\mbR^D$. We assume that those points are mathematically structured, say lying on a Riemannian manifold $\mcM$ embedded in $\mbR^D$, possibly perturbed by random noise in the ambient space $\mbR^D$. Considering the dimensionality $d$ of $\mcM$, also called \emph{intrinsic dimensionality}, we look for a $d$-dimensional representation $\bfy_i$, $i = 1, \ldots, n$ of those data. It is crucial to preserve some geometric structures or innate properties of the input data $\bfx_i$ in $\bfy_i$. In general, adjacent inputs needed to be mapped to adjacent outputs, while distant inputs are mapped to distant outputs. Denoting the geodesic distance of $\mcM$ by $d_\mcM(\cdot,\cdot)$, we require
\begin{equation*}
  d_\mcM(\bfx_i, \bfx_j) = \Vert \bfy_i - \bfy_j \Vert\quad\text{or}\quad d_\mcM(\bfx_i, \bfx_j) \approx \Vert \bfy_i - \bfy_j \Vert  \quad
\end{equation*}
for any  $i,\,j$. Here, $\Vert \cdot \Vert$ denotes the Euclidean distance.

With the notation $d_{ij} := d_\mcM(\bfx_i,\bfx_j)$, a classical dimensional reduction method, metric multidimensional scaling (mMDS) seeks
\begin{equation*}
  \min_{\bfy_1,\ldots,\bfy_n} \sum_{i < j} \big(d_{ij} - \Vert \bfy_i - \bfy_j \Vert\big)^2.
\end{equation*}
There are also various models related to this formulation, such as Principal Component Analysis (PCA), Isomap, Kernel PCA, Maximum Variance Unfolding, diffusion maps, etc. These models usually apply convex optimization techniques such as eigenvalue problem~\cite{vandermaaten2009}.

The primary task in formation control is to achieve a predetermined spatial configuration to a team of agents~\cite{Ahn20}. In distance-based formation control~\cite{OhParkAhn15, OhAhn14}, considering a given graph structure $\mcG = (\mcV, \mcE)$ with the vertex set $\mcV$ and the edge set $\mcE$, and a given \emph{realization} $p^* = (p^*_1,\ldots, p^*_n)\in \mbR^{nd}$, the \emph{desired formation} $E_p^*$ of the agents is defined as the set of formations that are \emph{congruent} to $p^*$:
\begin{equation*}
  E_{p^*} = \{(p_1,\ldots, p_n) \in \mbR^{nd}: \Vert p_i - p_j \Vert = \Vert p^*_i - p^*_j \Vert, \;i, j \in \mcV\}.
\end{equation*}
This set can be adjusted to the given graph structure:
\begin{equation*}
  E'_{p^*} = \{(p_1,\ldots, p_n) \in \mbR^{nd}: \Vert p_i - p_j \Vert = \Vert p^*_i - p^*_j \Vert, \;(i, j) \in \mcE\}.
\end{equation*}
In this case, a formation in this set is said to be \emph{equivalent} to $(\mcG, p^*)$.

To achieve a desired formation, gradient control laws have been widely employed~\cite{OhParkAhn15}. For the agent $i$, a local potential function can be defined as
\begin{equation*}
  \phi_i(p_i) = \frac{k_p}{2}\sum_{j\in\mcN_i} \gamma(\Vert p_i - p_j\Vert),
\end{equation*}
where $\mcN_i$ is the index set of neighbor agents of $p_i$, $k_p > 0$ and $\gamma: \mbR \rightarrow  \bar{\mbR}_+$ is a differentiable function. The potential $\gamma$ can be chosen as~\cite{KrickBrouckeFrancis09}:
\begin{equation*}
  \gamma(\Vert p_i - p_j\Vert) = \big(\Vert p_i - p_j\Vert^2 - \Vert p^*_i - p^*_j\Vert^2\big)^2.
\end{equation*}
To minimize the potential $\phi_i$, the gradient flow can be applied. The position dynamics of the agents, called the \emph{single-integrator} model is given as~\cite{OhParkAhn15, Ahn20}:
\begin{equation*}
  \dot{p_i} = -\nabla_{p_i}\phi_i, \; i=1,\ldots,n.
\end{equation*}

By regarding $\bfx_i$ and $\bfy_i$ as $p^*_i$ and $p_i$, respectively, we notice that the classical multidimensional scaling and the single-integrator model of the distance-based formation control seek a similar goal. The major difference lies in dimensionality. Nevertheless, motivated by formation control, we develop a nonlinear dynamical system for dimensionality reduction, which is given in the next section.

\section{Proposed Model}
We consider the following objective as a potential for dimensionality reduction:
\begin{equation}\label{model:local}
  \phi(\bfy_1,\ldots,\bfy_n) = \frac{1}{2}\sum_{i=1}^n \phi_i = \frac{1}{2}\sum_{i=1}^n\frac{1}{2}\sum_{j\in\mcN_i} \big|d_\mcM(\bfx_i,\bfx_j)^p - \Vert \bfy_i - \bfy_j\Vert^p\big|^q
\end{equation}
where $\mcN_i$ is the index set of the \emph{neighbor} points of $\bfx_i$ excluding $i$. The set $\mcN_i$ may consist of the $k$-nearest points or the points $\bfx_j$ in the $\epsilon$-neighborhood of $\bfx_i$, under the geodesic distance. We choose $p=1$ and $q=2$, unless otherwise specified. In this case, the selection of the exponent $p$ is to preserve the closer neighbor points more strongly, and the choice of $q$ is for an easy computation.

The difference from mMDS is that mMDS considers all pairwise distances, but the proposed potential takes only neighbors into account. Formation control seeks the equality in pairwise distances, but our case looks for approximates by finding a minimizer of the potential. To find it, similar to the case of formation control, we apply the gradient flow of $\phi_i$:
\begin{equation}\label{model:gradflow}
  \frac{d\bfy_i}{dt}  = -\nabla_{\bfy_i}\phi_i =  \sum_{j \in \mcN_i} \big(d_\mcM(\bfx_i, \bfx_j) - \Vert \bfy_i - \bfy_j \Vert\big)\frac{\bfy_i - \bfy_j}{\Vert \bfy_i - \bfy_j \Vert}, \; i=1,\ldots,n.
\end{equation}

We remark that if $\mcN_i$ is symmetric, i.e. $j\in\mcN_i$ implies $i\in\mcN_j$, for example, the points $\bfx_j$ in the $\epsilon$-neighborhood of $\bfx_i$, $\nabla_{\bfy_i}\phi_i = \nabla_{\bfy_i}\phi$, so that the equation \eqref{model:gradflow} consists of the gradient flow of \eqref{model:local}. Considering convergence, even in the case of the formation control, for $p$ to converge to $E'_{p^*}$ asymptotically, it is necessary that the graph $\mcG$ is a tree and any agents are not collocated initially~\cite{DimarogonasJohansson10, OhParkAhn15}. For example, let $3$ points $\bfx_1, \bfx_2, \bfx_3$ be sequentially located in a line with distance $1$, such as $(0, 0), (1, 0), (2, 0)$ in $\mbR^2$, respectively. With $\mcN_1 = \{2\},\, \mcN_2 = \{1, 3\}, \,\mcN_3 = \{2\}$, if $\bfy_1,\, \bfy_2,\, \bfy_3$ are initially given as $0,1,0$, respectively in $\mbR$, they form an equilibrium of \eqref{model:gradflow}. Thus, to achieve a desired formation, a proper initial guess and an appropriate graph structure is required.

Since the dynamical system~\eqref{model:gradflow} is controlled by the pairwise distances in neighbors i.e. \emph{local geometry}, the \emph{full geometry} may not be recovered. This may happen if the initial guess is far away from a desired formation, such as a random initial guess. Nonetheless, the reason why not enforcing all the pairwise distances, i.e. the full geometry is as follows. First, if $\bfx_j$ is not close enough to $\bfx_i$, the evaluation of $d_\mcM(\bfx_i, \bfx_j)$ is difficult. We note that if the point $\bfx_j$ is sufficiently close, $d_\mcM(\bfx_i, \bfx_j)$ can be well-approximated by $\Vert \bfx_i - \bfx_j \Vert$, on the contrary. Second, it is computationally cumbersome, especially on big data.

To circumvent the insufficiency of the geometry by neighbor points, we need to provide clues for global geometry. To emulate the full geometry more precisely, we add a non-local distance term in \eqref{model:gradflow}:
For $i=1, \ldots, n$ and $t> 0$,
\begin{equation}\label{model:full}
\begin{split}
\frac{d\bfy_i}{dt} = &\sum_{j \in \mcN_i} \big(d_\mcM(\bfx_i, \bfx_j) - \Vert \bfy_i - \bfy_j \Vert\big)\frac{\bfy_i - \bfy_j}{\Vert \bfy_i - \bfy_j \Vert} \\
& \quad + \lambda_t\sum_{j \in \mcR_i} \big(\tild_\mcM(\bfx_i, \bfx_j)-\Vert \bfy_i - \bfy_j \Vert \big)_+\frac{\bfy_i - \bfy_j}{\Vert \bfy_i - \bfy_j   \Vert}
\end{split}
\end{equation}
where $\lambda_t > 0$. We explain the second term in detail. The index set $\mcR_i$ is a subset of the complement of $\mcN_i \cup \{i\}$, i.e. $\mcR_i \subseteq \{1,\ldots,n\}\setminus (\mcN_i \cup \{i\})$ and $\tild_\mcM(\bfx, \bfy)$ is a lower bound or an approximate to the geodesic distance between two points $\bfx, \bfy$ in $\mcM$. Here, $x_+$ denotes $\max(0, x)$. We call $\mcR_i$ the set of \emph{remote} points of $\bfx_i$, and it can be a random subset of $\{1,\ldots,n\}\setminus (\mcN_i \cup \{i\})$. For $\tild_\mcM(\bfx, \bfy)$, the Euclidean distance $\Vert \bfx_i - \bfx_j \Vert$ or the length of shortest path in the graph by the neighbor points can be adopted.

Roughly speaking, the first term moves the point $\bfy_i$ relative to the neighbor points $\bfy_j$ to meet $\Vert \bfy_i - \bfy_j \Vert = d_\mcM(\bfx_i, \bfx_j)$ and the second term moves the point $\bfy_i$ away from the remote points $\bfy_j$ until $\Vert \bfy_i - \bfy_j \Vert \ge \tild_\mcM(\bfx_i, \bfx_j)$. Hence the first term governs the local geometry acting as configuration force, and the second term controls the global geometry acting as repulsive force. However, the the second term involves imprecise knowledge, so we may choose $\lambda_t \rightarrow 0$ as $t\rightarrow \infty$.

\section{Mathematical Properties of the Model}

In this section, we study some basic mathematical properties of our model. To shorten and clarify computations, we simplify $d_\mcM(\bfx_i, \bfx_j)$ and $\tild_\mcM(\bfx_i,\bfx_j)$ by $d_{ij}$ and $\tild_{ij}$, respectively. In addition, to bypass non-differentiability in \eqref{model:full}, we analyze a mollified version of the equation \eqref{model:full}:
\begin{equation}\label{model:mollified}
    \frac{d\bfy_i}{dt} = \sum_{j \in \mcN_i} (d_{ij} - \Vert \bfy_i - \bfy_j \Vert)\frac{\bfy_i - \bfy_j}{\Vert \bfy_i - \bfy_j \Vert} + \lambda\sum_{j \in \mcR_i} f_\varepsilon(\tild_{ij}-\Vert \bfy_i - \bfy_j \Vert)\frac{\bfy_i - \bfy_j}{\Vert \bfy_i - \bfy_j   \Vert}
\end{equation}
where $f_{\varepsilon}:\mathbb{R}\to\mathbb{R}$
is a nonnegative $C^{\infty}$-mollifier of the function $f(x)= x_+ = \max(0, x)$
such that
\[
\begin{cases}
f_{\varepsilon}(x)=f(x), & \left|x\right|>\varepsilon,\\
f(x)\leq f_{\varepsilon}(x)\leq\frac{1}{2}\left(x+\varepsilon\right) & \left|x\right|\leq\varepsilon.
\end{cases}
\]
We also restrict $0< \lambda \le 1$ to avoid an uninteresting constant for upper bound.

We assume that $\mcN_i$ and $\mcR_i$ are symmetric, $j\in\mcN_i$ implies $i\in\mcN_j$, for example.
With an initial condition $\bfy_i(0) = \bfy_{i}^{\text{in}}$, we have the following invariant property.

\begin{lem}\label{lemma:invarient}
  Let $\bfy_i, i = 1, \ldots, n$ be a solution to \eqref{model:mollified}.
  Then the model \eqref{model:mollified} is invariant under rigid motion. More precisely, for
  \[
  \tilde{\bfy}_i:=\Omega\bfy_i+\mathbf{c}
  \]
  for some $\Omega\in O(d)$ and $\mathbf{c}\in\mathbb{R}^{d}$,
  we have
  \begin{equation} \label{B-2-1}
  \begin{aligned}\frac{d\tilde{\bfy}_i}{dt} & =\sum_{j\in\mathcal{N}_{i}}\left(d_{ij}-\left\Vert \tilde{\bfy}_i-\tilde{\bfy}_j\right\Vert \right)\frac{\tilde{\bfy}_i-\tilde{\bfy}_j}{\left\Vert \tilde{\bfy}_i-\tilde{\bfy}_j\right\Vert }\\
   & \qquad+\lambda\sum_{j\in\mathcal{R}_{i}}f_{\varepsilon}\left(\tild_{ij}-\left\Vert \tilde{\bfy}_i-\tilde{\bfy}_j\right\Vert \right)\frac{\tilde{\bfy}_i-\tilde{\bfy}_j}{\left\Vert \tilde{\bfy}_i-\tilde{\bfy}_j\right\Vert }.
  \end{aligned}
  \end{equation}
  Furthermore, if $\bfy_i^{\text{in}}, \,i = 1, \ldots, n$ satisfy
  \begin{equation}\label{B-2-2}
    \bfy_{c}^{\text{in}}=\frac{1}{n}\left(\bfy_{1}^{\text{in}}+\cdots+\bfy_{n}^{\text{in}}\right)=\mathbf{0},
  \end{equation}
  then we have
  \[
    \bfy_{c}=\frac{1}{n}\left(\bfy_{1}+\cdots+\bfy_{n}\right)=\mathbf{0}.
  \]

  \end{lem}

  \begin{proof} We can obtain \eqref{B-2-1} by substituting
    \[
      \bfy_{i}=\Omega^{t}\tilde{\bfy}_i-\Omega^{t}\mathbf{c}
    \]
    in \eqref{model:mollified}.

    Now assume that \eqref{B-2-2} holds. Then
    \begin{equation}\label{B-2-4}
    \begin{aligned}\frac{d}{dt}\bfy_{c} & =\frac{1}{n}\left(\frac{d}{dt}\bfy_{1}+\cdots+\frac{d}{dt}\bfy_{n}\right)\\
     & =\frac{1}{n} \sum_{i=1}^n \sum_{j\in\mathcal{N}_{i}}\left(d_{ij}-\left\Vert \tilde{\bfy}_i-\tilde{\bfy}_j\right\Vert \right)\frac{\tilde{\bfy}_i-\tilde{\bfy}_j}{\left\Vert \tilde{\bfy}_i-\tilde{\bfy}_j\right\Vert }\\
   & \qquad+\frac{\lambda}{n} \sum_{i=1}^n \sum_{j\in\mathcal{R}_{i}}f_{\varepsilon}\left(\tild_{ij}-\left\Vert \tilde{\bfy}_i-\tilde{\bfy}_j\right\Vert \right)\frac{\tilde{\bfy}_i-\tilde{\bfy}_j}{\left\Vert \tilde{\bfy}_i-\tilde{\bfy}_j\right\Vert },
    \end{aligned}
    \end{equation}
    and since
    \[
    j\in\mathcal{N}_{i}\,\,\Longleftrightarrow\,\, i\in\mathcal{N}_{j}\quad\text{and}\quad j\in\mathcal{R}_{i}\,\,\Longleftrightarrow\,\, i\in\mathcal{R}_{j}
    \]
    hold, interchanging $i$ and $j$ in \eqref{B-2-4} gives
    \[
      \frac{d}{dt}\bfy_{1}+\cdots+\frac{d}{dt}\bfy_{n}=\mathbf{0}.
    \]
    \end{proof}

By Lemma \ref{lemma:invarient}, we assume that $\bfy_i, \,i = 1, \ldots, n$ are centered at origin:
\[
\frac{1}{n}\left(\bfy_{1}+\cdots+\bfy_{n}\right)\equiv\mathbf{0}.
\]
Next two lemmas show that the model \eqref{model:mollified} is a gradient flow. We use the following notation:
\[
\phi^\varepsilon_{i} :=\frac{1}{2}\sum_{j\in\mathcal{N}_{i}}\left(d_{ij} - \Vert \bfy_i - \bfy_j \Vert\right)^{2}+\lambda\sum_{j\in\mathcal{R}_{i}}F_{\varepsilon}\left(\tild_{ij}-\Vert \bfy_i - \bfy_j \Vert\right).
\]

\begin{lem} Let $\bfy_i, i = 1, \ldots, n$ be a solution to \eqref{model:mollified}. Then our model \eqref{model:mollified} is a gradient flow with a potential function
  \begin{equation*}
  \phi_\varepsilon(t, \bfy_1,\ldots,\bfy_n)=\frac{1}{2}\sum_{i=1}^n\left[\frac{1}{2}\sum_{j\in\mathcal{N}_{i}}\left(d_{ij} - \Vert \bfy_i - \bfy_j \Vert\right)^{2}+\lambda\sum_{j\in\mathcal{R}_{i}}F_{\varepsilon}\left(\tild_{ij}-\Vert \bfy_i - \bfy_j \Vert\right)\right],
  \end{equation*}
  where $F_{\varepsilon}$ be the antiderivative of $f_{\varepsilon}$
  satisfying $F_{\varepsilon}(0)=0$.
  \end{lem}

  \begin{proof} We fix $i$ and compute $\nabla_{\bfy_{i}}\phi^\varepsilon_{j}$  for each $j$.

    For the case $i=j$, we have
    \begin{equation}\label{B-3-1}
      \nabla_{\bfy_{i}}\phi^\varepsilon_{i} = -\sum_{j \in \mcN_i} (d_{ij} - \Vert \bfy_i - \bfy_j \Vert)\frac{\bfy_i - \bfy_j}{\Vert \bfy_i - \bfy_j \Vert} - \lambda\sum_{j \in \mcR_i} f_{\varepsilon}(\tild_{ij}-\Vert \bfy_i - \bfy_j \Vert)\frac{\bfy_i - \bfy_j}{\Vert \bfy_i - \bfy_j   \Vert}.
  \end{equation}
  For $j\in\mathcal{N}_{i}$, we have
    \begin{equation}\label{B-3-2}
      \nabla_{\bfy_{i}}\phi^\varepsilon_{j}=\frac{1}{2}\nabla_{\bfy_{i}} \left(d_{ji} - \Vert \bfy_j - \bfy_i \Vert\right)^{2}=-\left(d_{ji} - \Vert \bfy_j - \bfy_i \Vert\right)\frac{\bfy_i - \bfy_j}{\Vert \bfy_i - \bfy_j \Vert},
    \end{equation}
 and for $j\in\mathcal{R}_{i}$,
    \begin{equation}\label{B-3-3}
      \nabla_{\bfy_{i}}\phi^\varepsilon_{j}=\nabla_{\bfy_{i}}F_{\varepsilon}\left(\tild_{ji}-\Vert \bfy_j - \bfy_i \Vert \right)=-f_{\varepsilon}\left(\tild_{ji}-\Vert \bfy_j - \bfy_i \Vert \right)\frac{\bfy_i - \bfy_j}{\Vert \bfy_i - \bfy_j \Vert}.
    \end{equation}
    Otherwise, we have
    \begin{equation}\label{B-3-4}
      \nabla_{\bfy_{i}}\phi^\varepsilon_{j} = \mathbf{0}.
    \end{equation}
    By combining \eqref{B-3-1}, \eqref{B-3-2}, \eqref{B-3-3} and \eqref{B-3-4},
    we conclude that

    \[
    -\nabla_{\bfy_{i}}\phi_\varepsilon=-2\nabla_{\bfy_{i}}\phi^\varepsilon_i=\frac{d}{dt}\bfy_{i}.
    \]
    \end{proof}


\begin{lem}\label{lem:decreasing}
  Let $\bfy_i, \,i = 1, \ldots, n$ be a solution to \eqref{model:mollified}. Then the potential $\phi_\varepsilon$ of our
  model \eqref{model:mollified} is nonincreasing with
  \[
  \frac{d}{dt}\phi_\varepsilon=-\sum_{i=1}^n\left\Vert \sum_{j\in\mathcal{N}_{i}}(d_{ij} - \Vert \bfy_i - \bfy_j \Vert)\frac{\bfy_i - \bfy_j}{\Vert \bfy_i - \bfy_j \Vert}+\lambda\sum_{j\in\mathcal{R}_{i}}f_{\varepsilon}(\tild_{ij}-\Vert \bfy_i - \bfy_j \Vert)\frac{\bfy_i - \bfy_j}{\Vert \bfy_i - \bfy_j \Vert}\right\Vert ^{2}\leq0.
  \]
\end{lem}

\begin{proof}
  Since $\bfy_i$ are a gradient flow, we have
  \[
  \frac{d}{dt}\phi_\varepsilon=\nabla_\bfy\phi_\varepsilon\cdot\frac{d\bfy }{dt}= -\left\Vert \dot{\bfy}\right\Vert ^{2}=-\sum_{i = 1}^n\left\Vert \dot{\bfy}_{i}\right\Vert^2 .
  \]
  where
  \begin{equation}\label{collectey}
    \bfy = (\bfy_i, \ldots, \bfy_n)\in \mbR^{nd}.
  \end{equation}
\end{proof}

Assuming that $\mcN_i$ is not an empty set for $1\le i \le n$, we consider a undirected graph $\mathcal{G} = (\mathcal{V}, \mathcal{E})$ constructed by the local neighbor system of $\mcN_i$, more specifically, consisting of the vertex set $\mathcal{V} = \{i\in \mathbb{N}\ |\  1\le i \le n\}$ and the edge set $\mathcal{E} = \{(i, j)\ |\ j\in \mcN_i\}$, and assume that $\mathcal{G}$ is connected.

We also assume that there exist constants $d_m, \,d_M > 0$ such that
\[d_m \le d_{ij}, \ \tilde{d}_{ij} \  \text{ for }  \ j\in \mcN_i \cup \mcR_i \quad \text{and} \quad d_{ij} \le d_M \ \text{ for }\ j\in\mathcal{N}_{i}, \]
and use the notation
\[
\max_{1\le i \le n}\left|\mathcal{N}_{i}\right|:= n_\mcN,\quad\max_{1\le i \le n}\left|\mathcal{R}_{i}\right|:=n_\mcR,
\]
where $| \cdot |$ indicates the cardinality.

\begin{prop}\label{prop:uniformbound}
  Let $\bfy_i, i = 1, \ldots, n$ be a solution to \eqref{model:mollified}. Then the set $\{\bfy_i: i = 1, \ldots, n\} \subset \mbR^d$
  is uniformly bounded for all $t > 0$.
\end{prop}

  \begin{proof}
  For $j\in\mathcal{N}_{i}$, $\left(d_{ij}-\left\Vert \bfy_{i}(t)-\bfy_{j}(t)\right\Vert \right)^{2}\leq\phi_\varepsilon(0)$ from $\phi_\varepsilon(t) \le \phi_\varepsilon(0)$, $t>0$,
  by Lemma \ref{lem:decreasing}. Thus,
\[ \left\Vert \bfy_{i}(t)-\bfy_{j}(t)\right\Vert \le d_{ij}+\phi_\varepsilon(0)^{1/2}\leq d_M+\phi_\varepsilon(0)^{1/2},\quad t>0. \]

  Since $\mathcal{G}$ is connected, for any $i$, $j$, there exists a sequence
  \[i=k_{1},\ldots,k_{l_{i,j}}=j,\]
  such that $k_{l}\in\mathcal{N}_{l+1}$.

  Finally, we have
  \begin{align*}
  \left\Vert \bfy_{i}\right\Vert  & \leq\left\Vert \bfy_{i}-\bfy_{c}\right\Vert +\left\Vert \bfy_{c}\right\Vert \leq\frac{1}{n}\sum_{j=1}^{n}\left\Vert \bfy_{i}-\bfy_{j}\right\Vert \leq\frac{1}{n}\sum_{j=1}^{n}\sum_{l=1}^{l_{i,j}-1}\left\Vert \bfy_{k_l}-\bfy_{k_{l+1}}\right\Vert \\
   & \leq\frac{1}{n}\sum_{j=1}^{n}n\left(d_M+\phi_\varepsilon(0)^{1/2}\right)=n\left(d_M+\phi_\varepsilon(0)^{1/2}\right).
  \end{align*}

\end{proof}

Since our system is a gradient flow, we can guarantee the system \eqref{model:mollified} is globally well-posed
under a proper initial condition.  Here we restate Picard-Lindel\"of theorem~\cite{Hartman02}.

\begin{lem} (Picard-Lindel\"of) Let $\boldsymbol{f},\bfy\in\mathbb{R}^{d}$;
$\boldsymbol{f}(t,\bfy)$ continuous on a region $R:=\left[t_{0},t_{0}+a\right]\times\left\{ \left\Vert \bfy-\bfy_{0}\right\Vert \leq b\right\} $
and uniformly Lipschitz continuous with respect to $\bfy$.
If $\left\Vert \boldsymbol{f}(t,\bfy)\right\Vert \leq M$
on $R$, then the system
\[
\bfy'=f(t,\bfy),\qquad\bfy(t_{0})=\bfy_{0}
\]
has a unique solution $\bfy=\bfy(t)$ on $[t_{0},t_{0}+\min(a,b/M)]$.
\end{lem}

\begin{thm}\label{thm:well-posed}
   For $\varepsilon<\varepsilon_{1}<\max\left\{\frac{\lambda d_m}{2\sqrt{(n_\mcN+n_\mcR)}},\; d_m\right\}$,
  assume that for $i = 1, \ldots, n$,
  \[
  \left|d_{ij}-\left\Vert \bfy_{i}^{\text{in}}-\bfy_{j}^{\text{in}}\right\Vert \right|<\varepsilon_{1},\ j\in\mathcal{N}_{i} \quad\text{and}\quad \Vert \bfy_i^{\text{in}} - \bfy_j^{\text{in}} \Vert - \tild_{ij} > \varepsilon_{1},\ j\in\mathcal{R}_{i}.
  \]
  Then, the system \eqref{model:mollified} is globally well-posed.
  \end{thm}

  \begin{proof}
    First, we prove that $\left\Vert \bfy_{i}-\bfy_{j}\right\Vert $
    has a positive lower bound $\delta_{1}$ for $j\in \mcN_i \cup \mcR_i$.
    For each $i = 1, \ldots, n$, we know that
    \[
    \phi_{i}^\varepsilon(t)\leq\phi_{i}^\varepsilon(0),
    \]
    which implies
    \begin{equation}
    \begin{aligned} & \frac{1}{2}\sum_{j\in\mathcal{N}_{i}}\left(d_{ij}-\left\Vert \bfy_{i}(t)-\bfy_{j}(t)\right\Vert \right)^{2}+\lambda\sum_{j\in\mathcal{R}_{i}}F_{\varepsilon}\left(\tilde{d}_{ij}-\left\Vert \bfy_{i}(t)-\bfy_{j}(t)\right\Vert \right)\\
     & \qquad\leq\frac{1}{2}\sum_{j\in\mathcal{N}_{i}}\left(d_{ij}-\left\Vert \bfy_{i}^{\text{in}}-\bfy_{j}^{\text{in}}\right\Vert \right)^{2}+\lambda\sum_{j\in\mathcal{R}_{i}}F_{\varepsilon}\left(\tilde{d}_{ij}-\left\Vert \bfy_{i}^{\text{in}}-\bfy_{j}^{\text{in}}\right\Vert \right)\\
     & \qquad\leq\frac{1}{2}\left|\mathcal{N}_{i}\right|\varepsilon_{1}^{2}+\frac{\lambda}{2}\left|\mathcal{R}_{i}\right|\left(\varepsilon_{1}^{2}+\varepsilon^{2}\right)\leq\left(\left|\mathcal{N}_{i}\right|+\left|\mathcal{R}_{i}\right|\right)\varepsilon_{1}^{2}.
    \end{aligned}
    \label{C-2-1}
    \end{equation}
    Then, for each $j\in\mathcal{N}_{i}$, $\left(d_{ij}-\left\Vert \bfy_{i}(t)-\bfy_{j}(t)\right\Vert \right)^{2}\leq2\left(\left|\mathcal{N}_{i}\right|+\left|\mathcal{R}_{i}\right|\right)\varepsilon_{1}^{2}$ by \eqref{C-2-1}, and thus,
    \[0<(1-\lambda/\sqrt{2})d_m<d_{ij}-\sqrt{2\left(\left|\mathcal{N}_{i}\right|+\left|\mathcal{R}_{i}\right|\right)}\varepsilon_{1}\leq\left\Vert \bfy_{i}(t)-\bfy_{j}(t)\right\Vert .\]
    Also for $j\in\mathcal{R}_{i}$, \eqref{C-2-1} implies
    \begin{equation}
    \lambda F_{\varepsilon}\left(\tilde{d}_{ij}-\left\Vert \bfy_{i}(t)-\bfy_{j}(t)\right\Vert \right)\leq\left(\left|\mathcal{N}_{i}\right|+\left|\mathcal{R}_{i}\right|\right)\varepsilon_{1}^{2}.\label{C-2-2}
  \end{equation}

  If $\tilde{d}_{ij}-\left\Vert \bfy_{i}(t)-\bfy_{j}(t)\right\Vert >\varepsilon$
  in \eqref{C-2-2},  $\left(\tilde{d}_{ij}-\left\Vert \bfy_{i}(t)-\bfy_{j}(t)\right\Vert \right)^{2}\leq\frac{2\left(\left|\mathcal{N}_{i}\right|+\left|\mathcal{R}_{i}\right|\right)}{\lambda}\varepsilon_{1}^{2}$ by the definition of $F_{\varepsilon}$, and so,
  \[ 0<d_m-\sqrt{\frac{2\left(n_\mcN+n_\mcR\right)}{\lambda}}\varepsilon_{1}\le\tilde{d}_{ij}-\sqrt{\frac{2\left(\left|\mathcal{N}_{i}\right|+\left|\mathcal{R}_{i}\right|\right)}{\lambda}}\varepsilon_{1}\leq\left\Vert \bfy_{i}(t)-\bfy_{j}(t)\right\Vert. \]
  If $\tilde{d}_{ij}-\left\Vert \bfy_{i}(t)-\bfy_{j}(t)\right\Vert \leq\varepsilon$
  in \eqref{C-2-2}, we have
  \[
  0<d_m-\varepsilon_{1}\leq \tilde{d}_{ij}-\varepsilon\leq\left\Vert \bfy_{i}(t)-\bfy_{j}(t)\right\Vert.
  \]
  In any case, we have positive lower bound $\delta_{1}$ of $\left\Vert \bfy_{i}(t)-\bfy_{j}(t)\right\Vert $ for $j\in \mcN_i \cup \mcR_i$, regardless of $i$, $j$ and $t$.

  Next, we prove local Lipschitz-continuity of the right hand side of \eqref{model:mollified}. We consider two sets $S_{1}(\delta_{1})$ and $S_{2}$ by
  \[S_{1}(\delta_{1}):=\left\{\bfy \in \mbR^{nd}: \left\Vert \bfy_{i}-\bfy_{j}\right\Vert \geq\delta_{1},  \text{ for all } i, j  \text{ such that } j\in \mathcal{N}_{i}\cup\mathcal{R}_{i} \right\} \]
  and
  \[S_{2}:=\left\{\bfy \in \mbR^{nd}: \bfy_{i}=\bfy_{j}, \text{ for some } i, j  \text{ such that }  j\in \mathcal{N}_{i}\cup\mathcal{R}_{i}\right\}, \]
where $\bfy$ is defined in \eqref{collectey}. Note that our initial position $\bfy^{\text{in}}$ lies in $S_{1}(\delta_{1})$.

Let $\bfy^{1} \in S_{1}(\delta_{1})$ and $\bfy^{2} \in S_2$. We can choose some $i$, $j$ such that $\bfy_i^2 = \bfy_j^2$ and $j\in \mathcal{N}_{i}\cup\mathcal{R}_{i}$, which implies
\begin{equation*}
  \left\Vert \bfy_{i}^{1}-\bfy_{i}^{2}\right\Vert ^{2}+\left\Vert \bfy_{j}^{1}-\bfy_{j}^{2}\right\Vert ^{2}\geq\frac{1}{2}\left\Vert \bfy_{i}^{1}-\bfy_{j}^{1}\right\Vert^2 \geq\frac{\delta_{1}^2}{2},
\end{equation*}
where the first inequality is by applying the parallelogram law. Thus,
\[
\left\Vert \bfy^{1}-\bfy^{2}\right\Vert \ge \sqrt{\left\Vert \bfy_{i}^{1}-\bfy_{i}^{2}\right\Vert ^{2}+\left\Vert \bfy_{j}^{1}-\bfy_{j}^{2}\right\Vert ^{2}} \geq\frac{\delta_{1}}{\sqrt{2}}.
\]

For $\bfy^{1}\in S_{1}(\delta_{1})$ and $\bfy^{2}\in B_{\delta_{1}/\sqrt{2}}(\bfy^{1})$, the open ball centered at $\bfy^{1}$ with radius $\delta_{1}/\sqrt{2}$, we have
\begin{align*}
 & \left\Vert \frac{\bfy_{i}^{1}-\bfy_{j}^{1}}{\left\Vert \bfy_{i}^{1}-\bfy_{j}^{1}\right\Vert }-\frac{\bfy_{i}^{2}-\bfy_{j}^{2}}{\left\Vert \bfy_{i}^{2}-\bfy_{j}^{2}\right\Vert }\right\Vert \\
 & =\left\Vert \frac{\bfy_{i}^{1}-\bfy_{j}^{1}}{\left\Vert \bfy_{i}^{1}-\bfy_{j}^{1}\right\Vert }-\frac{\bfy_{i}^{2}-\bfy_{j}^{2}}{\left\Vert \bfy_{i}^{1}-\bfy_{j}^{1}\right\Vert }+\frac{\bfy_{i}^{2}-\bfy_{j}^{2}}{\left\Vert \bfy_{i}^{1}-\bfy_{j}^{1}\right\Vert }-\frac{\bfy_{i}^{2}-\bfy_{j}^{2}}{\left\Vert \bfy_{i}^{2}-\bfy_{j}^{2}\right\Vert }\right\Vert \\
 & \leq\left\Vert \frac{\left(\bfy_{i}^{1}-\bfy_{j}^{1}\right)-\left(\bfy_{i}^{2}-\bfy_{j}^{2}\right)}{\left\Vert \bfy_{i}^{1}-\bfy_{j}^{1}\right\Vert }\right\Vert +\left\Vert \frac{\bfy_{i}^{2}-\bfy_{j}^{2}}{\left\Vert \bfy_{i}^{1}-\bfy_{j}^{1}\right\Vert }-\frac{\bfy_{i}^{2}-\bfy_{j}^{2}}{\left\Vert \bfy_{i}^{2}-\bfy_{j}^{2}\right\Vert }\right\Vert \\
 & \leq\frac{2}{\delta_{1}}\left\Vert \bfy^{1}-\bfy^{2}\right\Vert +\frac{\left|\left\Vert \bfy_{i}^{1}-\bfy_{j}^{1}\right\Vert -\left\Vert \bfy_{i}^{2}-\bfy_{j}^{2}\right\Vert \right|}{\left\Vert \bfy_{i}^{1}-\bfy_{j}^{1}\right\Vert }\\
 & \leq\frac{2}{\delta_{1}}\left\Vert \bfy^{1}-\bfy^{2}\right\Vert +\frac{1}{\delta_{1}}\left\Vert (\bfy_{i}^{1}-\bfy_{j}^{1})-(\bfy_{i}^{2}-\bfy_{j}^{2})\right\Vert \leq\frac{4}{\delta_{1}}\left\Vert \bfy^{1}-\bfy^{2}\right\Vert,
\end{align*}
and this shows the local Lipschitz continuity of $\frac{\bfy_{i}-\bfy_{j}}{\left\Vert \bfy_{i}-\bfy_{j}\right\Vert}$ in $B_{\delta_{1}/\sqrt{2}}(\bfy^{1})$.

Since $\left(d_{ij}-\left\Vert \bfy_{i}-\bfy_{j}\right\Vert \right)\frac{\bfy_{i}-\bfy_{j}}{\left\Vert \bfy_{i}-\bfy_{j}\right\Vert }$ and $f_{\varepsilon}\left(d_{ij}-\left\Vert \bfy_{i}-\bfy_{j}\right\Vert \right)\frac{\bfy_{i}-\bfy_{j}}{\left\Vert \bfy_{i}-\bfy_{j}\right\Vert }$ are a
product of local Lipschitz functions, they are local Lipschitz, and so is \eqref{model:mollified}.

Finally, we apply Picard-Lindel\"of theorem iteratively to prove well-posedness.
By Proposition \ref{prop:uniformbound}, we have
\begin{align*}
& \quad \left\Vert \sum_{j\in\mathcal{N}_{i}}\left(d_{ij}-\left\Vert \bfy_{i}-\bfy_{j}\right\Vert \right)\frac{\bfy_{i}-\bfy_{j}}{\left\Vert \bfy_{i}-\bfy_{j}\right\Vert }+\lambda\sum_{j\in\mathcal{R}_{i}}f_{\varepsilon}\left(\tilde{d}_{ij}-\left\Vert \bfy_{i}-\bfy_{j}\right\Vert \right)\frac{\bfy_{i}-\bfy_{j}}{\left\Vert \bfy_{i}-\bfy_{j}\right\Vert }\right\Vert \\
 & \leq\sum_{j\in\mathcal{N}_{i}}\left(d_{ij}+\left\Vert \bfy_{i}\right\Vert +\left\Vert \bfy_{j}\right\Vert \right)+\lambda\sum_{j\in\mathcal{R}_{i}}\left(\tilde{d}_{ij} + \left\Vert \bfy_{i}\right\Vert +\left\Vert \bfy_{j}\right\Vert + \varepsilon\right) \\
 & \le \sum_{j\in\mathcal{N}_{i}} \left(d_{ij} + 2n\left(d_M+\phi_\varepsilon(0)^{1/2}\right)\right) +\lambda\sum_{j\in\mathcal{R}_{i}}\left(\tilde{d}_{ij} +  \varepsilon + 2n\left(d_M+\phi_\varepsilon(0)^{1/2}\right)\right)\\
 & \le C + 2n(n_\mcN + n_\mcR)\left(d_M+\phi_\varepsilon(0)^{1/2}\right) :=M,
\end{align*}
with a positive constants $C$. We can apply Picard-Lindel\"of
theorem with $0\leq t<\frac{\delta_{1}}{2\sqrt{2}M}:=t_{0}$. Hence
we find the unique continuous solution in $[0,t_{0}/2]$. Again by Proposition \ref{prop:uniformbound} and the first step, we have for any $i$ and $j\in \mcN_i \cup \mcR_i$,
\[
\left\Vert \bfy_i(t_{0}/2)\right\Vert \leq n\left(d_M+\phi_\varepsilon(0)^{1/2}\right) \quad \text{and} \quad
\left\Vert \bfy_{i}(t_{0}/2)-\bfy_{j}(t_{0}/2)\right\Vert \geq\delta_{1}.
\]
Hence we can reuse Picand-Lindel\"of theorem, which guarantees the
existence of unique solution to \eqref{model:mollified} in
$
t\in\left[\frac{t_{0}}{2},\frac{3t_{0}}{2}\right).
$
By following this process inductively, we obtain an unique solution in
$t\in[0,\infty)$.

\end{proof}

Theorem~\ref{thm:well-posed}, we can assert that there
is some initial condition that our model \eqref{model:mollified} is well-posed. Involing only neigbor points, the model~\eqref{model:gradflow} can be further analyzed.  We follow the paper~\cite{HaLiXue13}  to prove that $\bfy_{i}$ converges to a stationary limit $\bfy_{i}^\infty$, $i = 1,\ldots, n$. The following Lojasiewicz gradient inequality plays a key role.

\begin{thm} Let $U\subset\mathbb{R}^{n}$ be open and $f:U\to\mathbb{R}$
be analytic. Then for any $z_{0}\in U$, there exists constants $\gamma\in[\frac{1}{2},1),\,C_{L},\, r>0$
such that
\[
\left|f(z)-f(z_{0})\right|^{\gamma}\leq C_{L}\left\Vert \nabla f(z)\right\Vert ,\quad z\in B_{r}(z_{0})\subset U.
\]
\end{thm}

\begin{lem} Let $\bfy = [\bfy_i, \ldots, \bfy_n]\in \mbR^{nd}$
  be a solution to~\eqref{model:gradflow}. Then for some $T_{0}>0$, $\varepsilon\ll1$,
  $\bfy_{\infty}\in\mathbb{R}^{nd}$, we have
  \[
  \bfy(t)\in B_{\varepsilon}(\bfy_{\infty}),\quad t\geq T_{0}.
  \]
\end{lem}

\begin{proof}
Similar to Proposition~\ref{prop:uniformbound}, the trajectory $\{\bfy(t)\}_{t\ge 0} $ is bounded, so we can choose
\begin{equation*}
  t_{n}\nearrow\infty,\quad\lim_{n\to\infty}\bfy(t_{n})=\bfy_\infty.
\end{equation*}
By the monotone property, we have
  \[
    \phi(\bfy(t_{n}))\searrow\phi_{\infty}.
  \]
Therefore
  \[
    \phi(\bfy_\infty)=\phi_{\infty}.
  \]

By Lojasiewicz's inequality at $\bfy=\bfy_{\infty}$, there exists some $\gamma\in[\frac{1}{2},1),\ C_{L}>0$, $r>0$ such that
  \[
    \left|\phi(\bfy)-\phi(\bfy_{\infty})\right|^{\gamma}\leq C_{L}\left\Vert \nabla\phi(\bfy)\right\Vert,\quad\bfy\in B_{r}(\bfy_{\infty}).
  \]
Consider an auxiliary function
  \begin{equation}\label{C-3-2}
    f(t)=\left(\phi(\bfy)-\phi(\bfy_{\infty})\right)^{1-\gamma},
  \end{equation}
then we have
  \[
    f(t)\searrow0\quad\text{as}\quad t\to\infty.
  \]
Therefore, for sufficiently small $\varepsilon\in(0,r)$, there exists $T_{0}>0$ such that
  \begin{equation}
    \left|f(t)-f(T_{0})\right|\leq\frac{\varepsilon(1-\gamma)}{3C_{L}},\quad t\geq T_{0}.\label{C-3-3}
  \end{equation}
Furthermore, we can select $T_{0}$ to satisfy
  \begin{equation}\label{C-3-4}
    \left\Vert \bfy(T_{0})-\bfy_{\infty}\right\Vert \leq\frac{\varepsilon}{3}.
  \end{equation}

Now, suppose that there exists $t_{1}>T_{0}$ such that
  \[
    \bfy(t_{1})\in B_{\varepsilon}(\bfy_{\infty})^{c}.
  \]
    Let $T_{1}$ to be the first exit time from the region $B_{\varepsilon}(\bfy_{\infty}):$
    \[
    T_{1}:=\inf\left\{ t\geq T_{0}:\bfy(t)\notin B_{\varepsilon}(\bfy_{\infty})\right\},
    \]
    which leads to
    \begin{equation}
    \left\Vert \bfy(T_{1})-\bfy_{\infty}\right\Vert =\varepsilon.\label{C-3-5}
    \end{equation}
    By \eqref{C-3-2} and the Lojasiewics's inequality, we have
    \begin{equation}
    \begin{aligned}\frac{df}{dt} & =\left(1-\gamma\right)\left(\phi(\bfy)-\phi(\bfy_{\infty})\right)^{-\gamma}\frac{d}{dt}\phi(\bfy(t))\\
     & =-\left(1-\gamma\right)\left(\phi(\bfy)-\phi(\bfy_{\infty})\right)^{-\gamma}\left\Vert \nabla\phi(\bfy(t))\right\Vert ^{2}\\
     & \leq-\frac{1-\gamma}{C_{L}}\left\Vert \nabla\phi(\bfy(t))\right\Vert ,\quad t\in[T_{0},T_{1}].
    \end{aligned}
    \label{C-3-6}
    \end{equation}

    We combine \eqref{C-3-3} and \eqref{C-3-6} to conclude
    \begin{equation}
    \int_{T_{0}}^{t}\left\Vert \frac{d}{dt}\bfy(s)\right\Vert ds=\int_{T_{0}}^{t}\left\Vert \nabla\phi(\bfy(t))\right\Vert ds\leq-\frac{C_{L}}{1-\gamma}(f(t)-f(T_{0}))\leq\frac{\varepsilon}{3},\quad t\in[T_{0},T_{1}].\label{C-3-7}
    \end{equation}
    On the other hand, by \eqref{C-3-4} and \eqref{C-3-7}, we obtain
    \begin{align*}
    \left\Vert \bfy(T_{1})-\bfy_{\infty}\right\Vert  & \leq\left\Vert \bfy(T_{1})-\bfy(T_{0})+\bfy(T_{0})-\bfy_{\infty}\right\Vert \\
     & \leq\int_{T_{0}}^{T_{1}}\left\Vert \frac{d}{dt}\bfy(s)\right\Vert ds+\left\Vert \bfy(T_{0})-\bfy_{\infty}\right\Vert \leq\frac{2\varepsilon}{3},
    \end{align*}
    contradicts to \eqref{C-3-5}.

    \end{proof}

\begin{thm} Let $\bfy=\left\{ \bfy_{i}\right\} _{i = 1, \ldots, n}$ be a solution to~\eqref{model:gradflow}. Then we have
\begin{enumerate}[(i)]
  \item $\bfy_{i}\to\bfy_{i}^{\infty},\quad i = 1, \ldots, n$ for some $\bfy_{i}^{\infty}\in\mathbb{R}^{d}$.
  \item $\dot{\bfy}_{i}\to \mathbf{0},\quad i = 1, \ldots, n.$
  \item $e_{ij}:=d_{ij} - \Vert \bfy_i - \bfy_j \Vert\to e_{ij}^{\infty},\quad j\in\mathcal{N}_{i}$ for some $e_{ij}^{\infty}\in\mathbb{R}_{\geq0}$.
\end{enumerate}
\end{thm}

\begin{proof} (i) We modify \eqref{C-3-7} to derive
\begin{equation}
\int_{T_{0}}^{\infty}\left\Vert \frac{d}{dt}\bfy(s)\right\Vert ds\leq\frac{\varepsilon}{3},\quad t\in[T_{0},\infty),\label{C-4-1}
\end{equation}
which implicates that the trajectory $\bfy(t)$ has a finite
length and it converges:
\[
\lim_{t\to\infty}\bfy(t)=\bfy_{\infty}.
\]
(ii) Since the model~\eqref{model:gradflow} is a gradient flow, we have
\[
\nabla_{\bfy}\phi(\bfy_{\infty})= \bfzr
\]
and
\[
\dot{\bfy}_{i}\to0,\quad i = 1, \ldots, n.
\]
(iii) By direct calculation,
\[
\left|\frac{d}{dt}e_{ij}\right|=\left|\left\langle \frac{\bfy_{i}-\bfy_{j}}{\left\Vert \bfy_{i}-\bfy_{j}\right\Vert },\,\dot{\bfy}_{i}-\dot{\bfy}_{j}\right\rangle \right|\leq\left\Vert \dot{\bfy}_{i}-\dot{\bfy}_{j}\right\Vert \leq\left\Vert \dot{\bfy}_{i}\right\Vert +\left\Vert \dot{\bfy}_{j}\right\Vert
\]
holds and by \eqref{C-4-1},
\[
\left|\int_{0}^{t}\frac{d}{ds}e_{ij}(s)ds\right|\leq\int_{0}^{t}\left|\frac{d}{ds}e_{ij}(s)\right|ds\leq\int_{0}^{\infty}\left\Vert \dot{\bfy}_{i}(s)\right\Vert +\left\Vert \dot{\bfy}_{j}(s)\right\Vert ds<\infty,\quad t\geq0.
\]
Hence $\frac{d}{dt}e_{ij}$ is integrable and $e_{ij}$ is converges.

\end{proof}

\section{Computational Scheme and Experiments}
This section considers the computational scheme of the model and presents numerical experiments. Overall, we utilize the forward Euler method to solve the dynamical system.

Under the assumption of Riemannian manifold $\mcM$ in $\mbR^m$, we approximate the geodesic distance $d_\mcM(\bfx_i, \bfx_j)$ by the Euclidean distance $\Vert \bfx_i - \bfx_j \Vert$ in $\mbR^m$ for the neighbor point $\bfx_j$ of $\bfx_i$, denoted as $d_{ij}$. For $\tild_\mcM(\bfx_i,\bfx_j)$ of the remote point $\bfx_j$, one may adopt $\Vert \bfx_i - \bfx_j \Vert$ in $\mbR^m$ as a lower bound, or apply Dijkstra's algorithm to the graph generated by the neighbor points to approximate $d_\mcM(\bfx_i, \bfx_j)$. In the experiments, we apply the latter, unless otherwise specified. Using Dijkstra's algorithm is particularly useful for curvy manifolds. Similarly, we re-express $\tild_\mcM(\bfx_i,\bfx_j)$ as $\tild_{ij}$. The computation scheme is a simple foward Euler scheme:
\begin{equation}\label{model:compu}
  \begin{split}
  \bfy_i^{t+\Delta t} = \bfy_i^t  + &\Delta t\sum_{j \in \mcN_i} \big(d_{ij} - \Vert \bfy_i^t - \bfy_j^t \Vert \big)\frac{\bfy_i^t - \bfy_j^t}{\Vert \bfy_i^t - \bfy_j^t \Vert + \delta} \\
  & \quad +  \Delta t\lambda_t\sum_{j \in \mcR_i} \big(\tild_{ij}-\Vert \bfy_i^t - \bfy_j^t \Vert \big)_+\frac{\bfy_i^t - \bfy_j^t}{\Vert \bfy_i^t - \bfy_j^t \Vert + \delta}.
  \end{split}
\end{equation}
Here, we add small $\delta > 0$ to avoid degeneracy and the superscript indicates time or iteration. We stop the iteration if $\max_{1\le i \le n} \Vert\bfy_i^{t+\Delta t} - \bfy_i^t\Vert < \eta$  for some $\eta > 0$.

As a preprocess, we normalize data to facilitate parameter selection. Let $\bar{\mu}(\cdot)$ and $\bar{\sigma}(\cdot)$ be the sample mean and the sample standard deviation of data, respectively. We normalize data as follows:
$$\frac{\bfx_i - \bar{\mu}(\bfx_i)}{\bar{\mu}(\Vert\bfx_i\Vert)+ \bar{\sigma}(\Vert\bfx_i\Vert)}.$$
Roughly speaking, after this process, the majority of the data are located within the ball of radius 2 centered at the origin. Now, we can choose uniform parameters $\delta = 10^{-7}$, $\Delta t = 0.2$, and $\eta = 10^{-3}$ for all experiments. We also set $\lambda_t = 1$. In our experience, there is no significant difference when letting $\lambda_t \rightarrow 0$.

For the neighbor point set $\mcN_i$, we use $k$-nearest points with $k = 20$ in the experiment, which we denote as kDRFC. In this case, there is no symmetry in the neighbor sets; $j\in\mcN_i$ does not guarantee $i\in\mcN_j$. For symmetry, one may consider an $\epsilon$-neighborhood. However, since it is also based on Euclidean distance, it can be challenging to find a suitable $\epsilon$. Some points may have too many neighbors, while others may be isolated. To ensure a similar number of neighbor points, first we choose the $k$-nearest points and then symmetrize the neighborhood. If $\bfx_j$ is within the $k$-nearest points of $\bfx_i$, we includes $i\in\mcN_j$ at the same time as $j\in\mcN_i$. In the experiments, we choose $k = 18$ and the average number of neighbor points is approximately $20$ after symmetrization. We denote this approch as sDRFC.

For the remote point set $\mcR_i$, we aim to select a similar number of points as in $\mcN_i$ from the set $\{1,\ldots,n\}\setminus (\mcN_i \cup \{i\})$. In the case of kDRFC, we randomly choose $20$ points for experiments. Similarly, for sDRFC, we randomly select around half of the average number of points in $\mcN_i$ and then symmetrize; if $j$ is in $\mcR_i$, we include $i$ in $\mcR_j$. Since random selections are usually disjoint, the number of points is doubled after symmetrization. Thus, we select $10$ points.

For initial guess, we perform an orthogonal projection onto a randomly selected $d$-dimensional subspace of $\mbR^m$ and use the resulting coordinate vectors. In addition, we add random noise using a Gaussian distribution. To generate an orthonormal basis for such a subspace, we randomly select $d$ vectors in $\mbR^m$ and apply QR decomposition to these vectors.

The cost of one iteration is $O(knd)$ for kDRFC and sDRFC has a similar computational cost. Since we only establish the local stability of the reduced model, we do not have the exact convergence rate. With the pre-specified parameters, typically $150$ to $300$ iterations are required to satisfy the stopping criterion on the synthetic data sets. We note the non-uniqueness of solutions and as a consequence, computational results may depend on the initial guess. Since the computational cost is not high, we conduct $5$ runs with different initial guesses and report the minimum value of~\eqref{model:local}.

We have tested our model on synthetic and real datasets benchmarked in~\cite{vandermaaten2009}. The synthetic data sets are the Swiss roll dataset, the helix dataset, the twin peaks dataset, the broken Swiss roll dataset, as shown in Figure~\ref{fig:synthetic}. All those datasets consist of $5,000$ samples, unless otherwise specified. The real data sets include the MNIST dataset, the COIL20 dataset, the ORL dataset, and the HIVA dataset.
The MNIST dataset consists of $60,000$ handwritten digits of size $28 \times 28$ pixels, from which $5,000$ digits are randomly selected for experiments.
The COIL20 dataset contains $32 \times 32$ images of $20$ different objects from $72$ viewpoints, totally $1,440$ images. The ORL dataset is a face recognition dataset of $400$ grayscale images with size $112 \times 92$ pixels that illustrates $40$ faces under various conditions. The HIVA dataset is a drug discovery dataset with two classes and consists of $3,845$ datapoints with dimensionality $1,617$. We refer to~\cite{vandermaaten2009} for more details.

\begin{figure}[t]
  \center{
  \subfigure[Swiss roll]
  {\includegraphics[width=5cm]{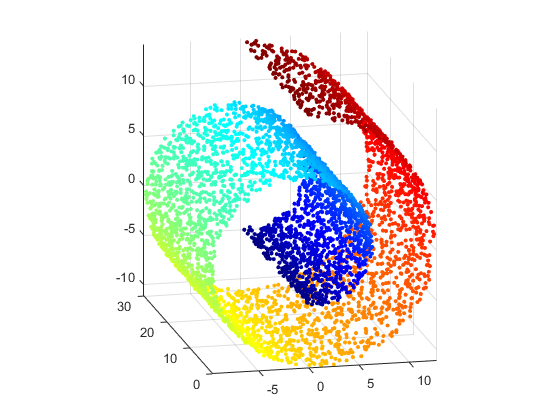}}
  \subfigure[Helix]
  {\includegraphics[width=5cm]{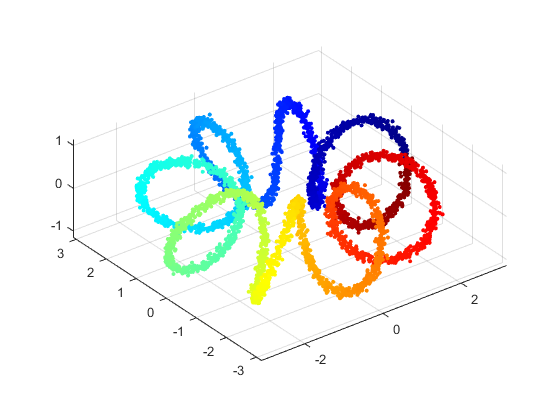}}
  \subfigure[Twin peaks]
  {\includegraphics[width=5cm]{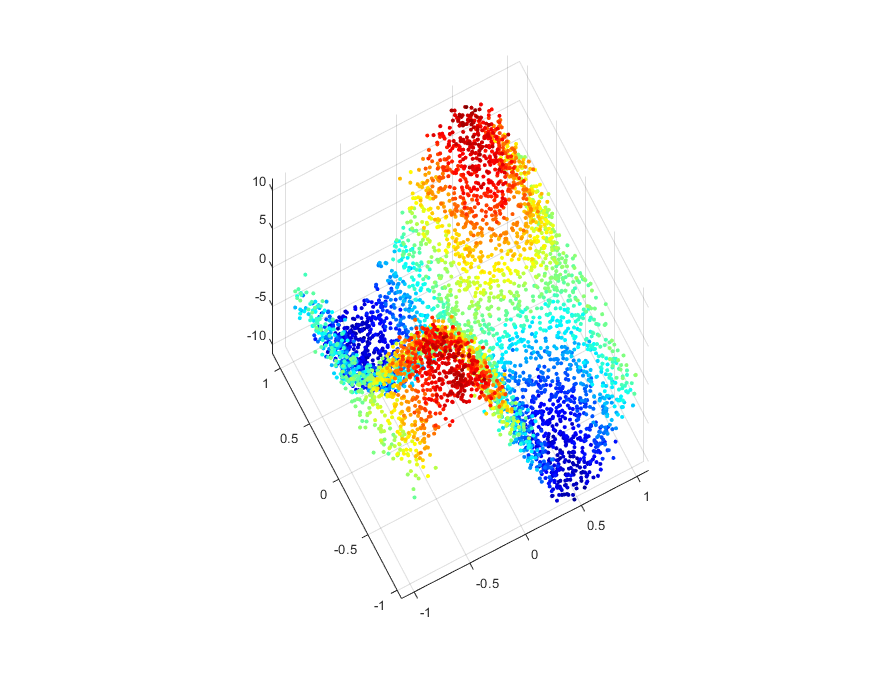}}
  \subfigure[Broken Swiss roll]
  {\includegraphics[width=5cm]{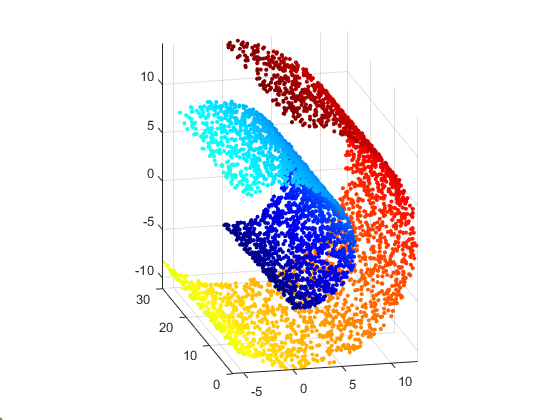}}
}
\caption{4 synthetic data sets}\label{fig:synthetic}
\end{figure}

Figures \ref{fig:swissroll} and \ref{fig:helix} show the progress of the scheme \eqref{model:compu} on the synthetic data sets, Swiss roll and helix. In Figure \ref{fig:swissroll}(a), the initial guess is displayed in the two-dimensional plane for the Swiss roll dataset. Figures \ref{fig:swissroll} (b), (c), and (d) present the results after $10$, $30$ and $100$ iterations, respectively. After around $100$ iterations, a rectangular shape clearly emerges, a flattened Swiss roll. The subfigures of Figure \ref{fig:helix} are similarly interpreted, and Figures \ref{fig:swissroll} (c) and (d) reveal a circle structure. We apply kDRFC for Figure \ref{fig:swissroll} and sDRFC for Figure \ref{fig:helix}.

\begin{figure}[t]
  \center{
  \subfigure[Initial guess]
  {\includegraphics[width=5cm]{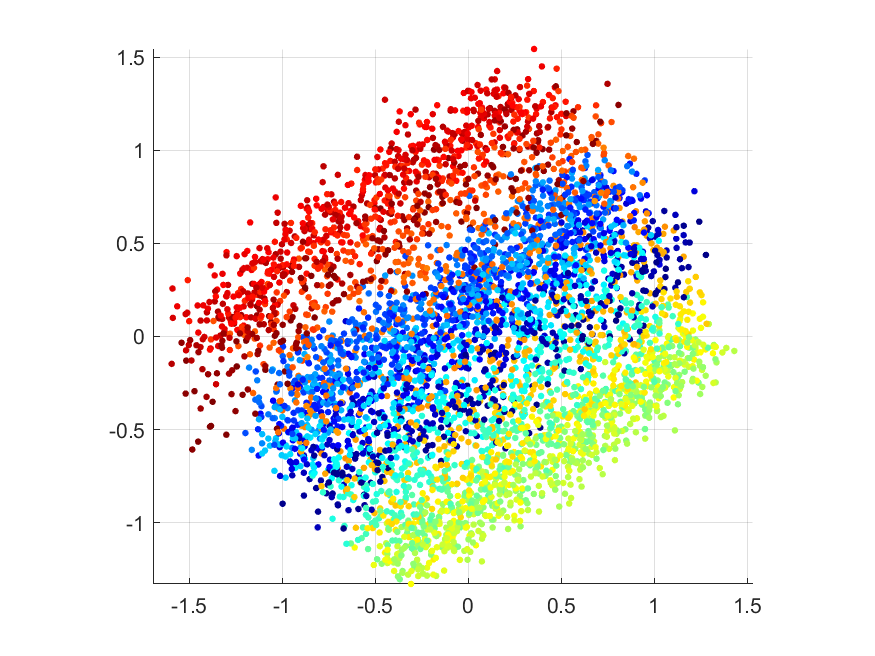}}
  \subfigure[10 iterations]
  {\includegraphics[width=5cm]{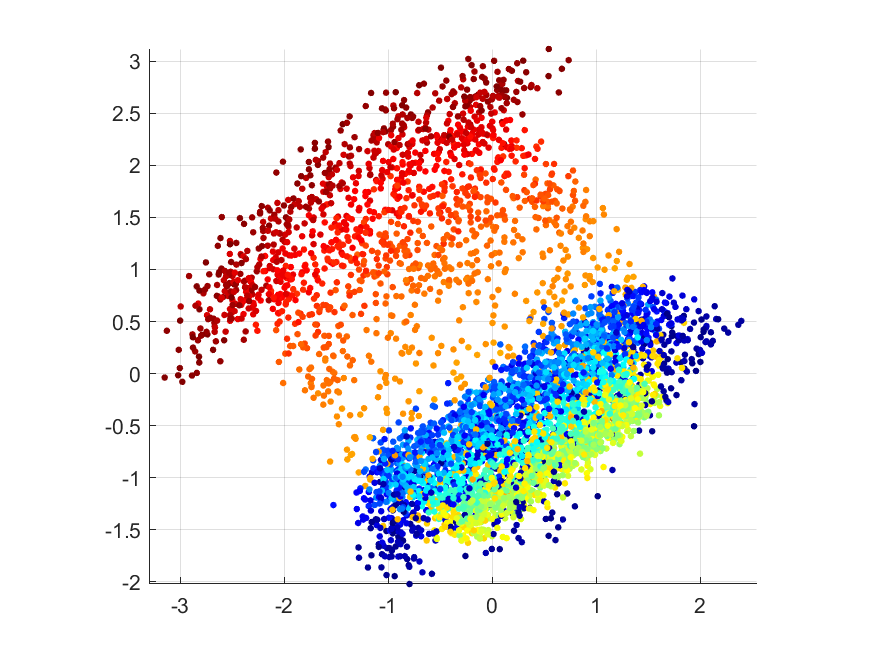}}
  \subfigure[30 iterations]
  {\includegraphics[width=5cm]{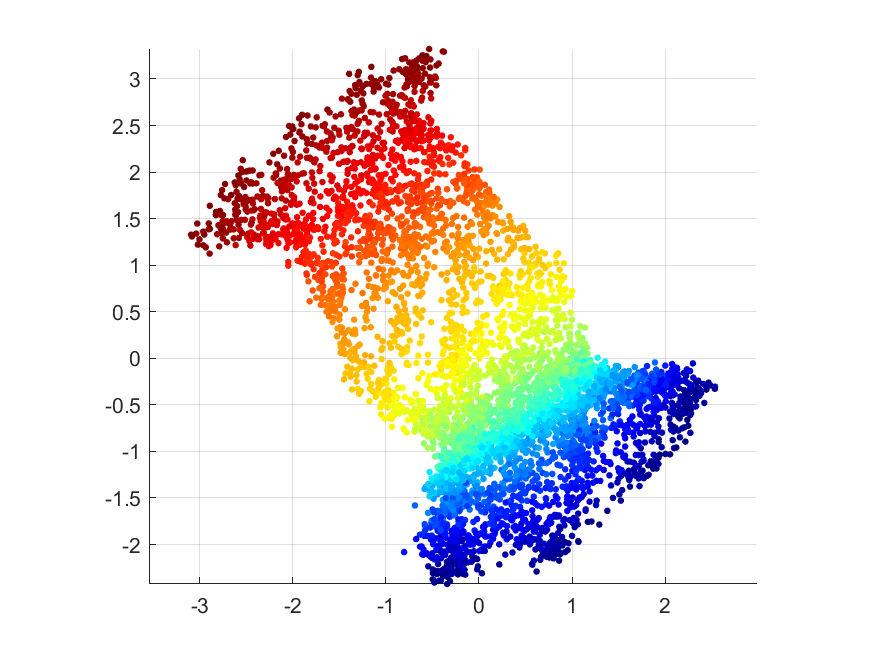}}
  \subfigure[100 iterations]
  {\includegraphics[width=5cm]{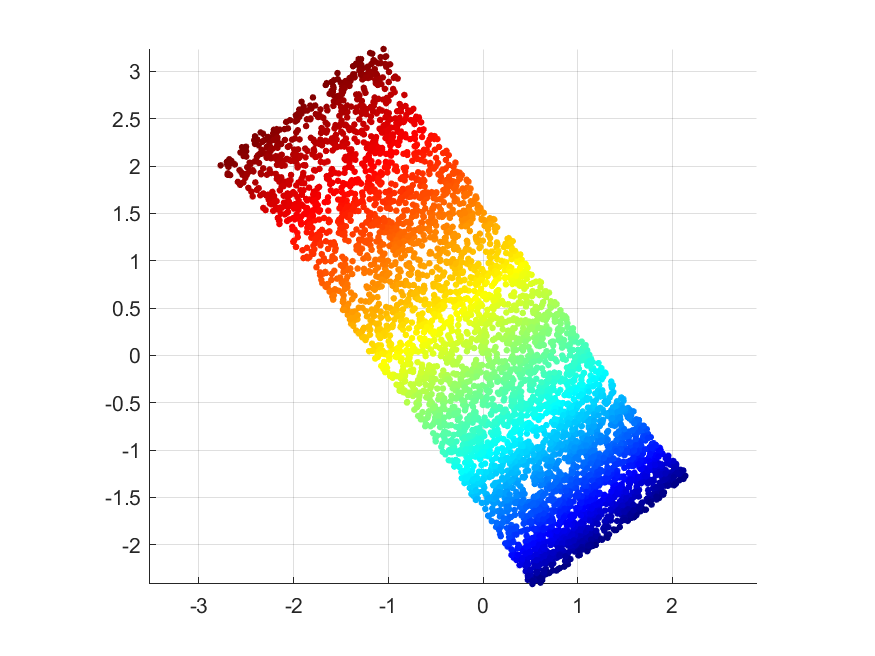}}
}
\caption{Computation of Swiss roll data using kDRFC}\label{fig:swissroll}
\end{figure}

\begin{figure}[t]
  \center{
  \subfigure[Initial guess]
  {\includegraphics[width=5cm]{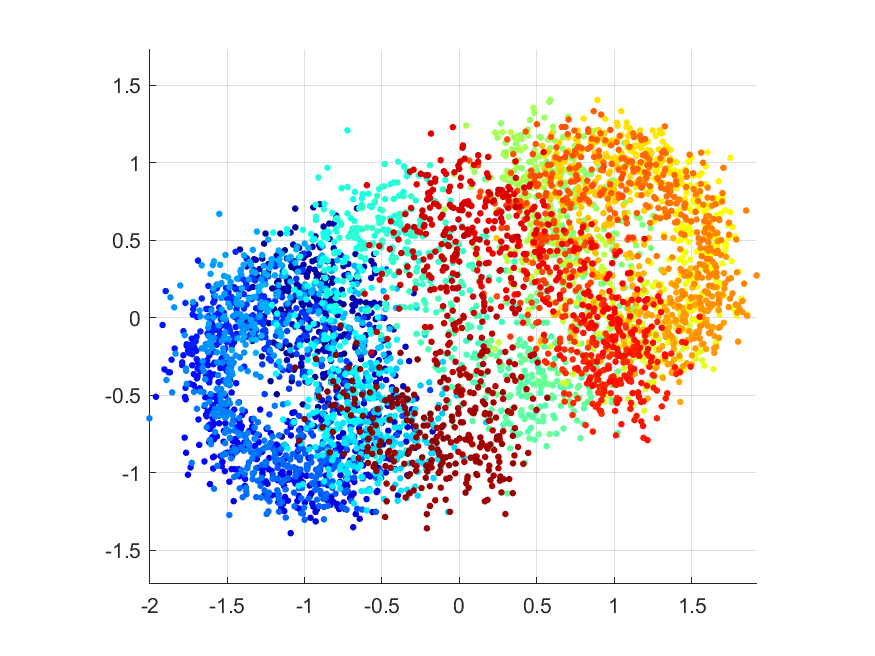}}
  \subfigure[40 iterations]
  {\includegraphics[width=5cm]{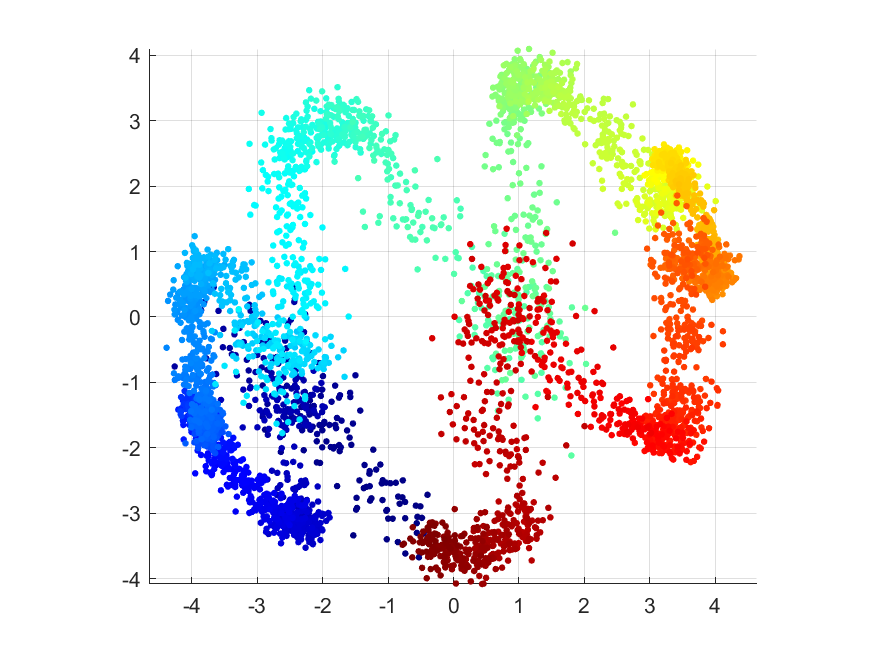}}
  \subfigure[80 iterations]
  {\includegraphics[width=5cm]{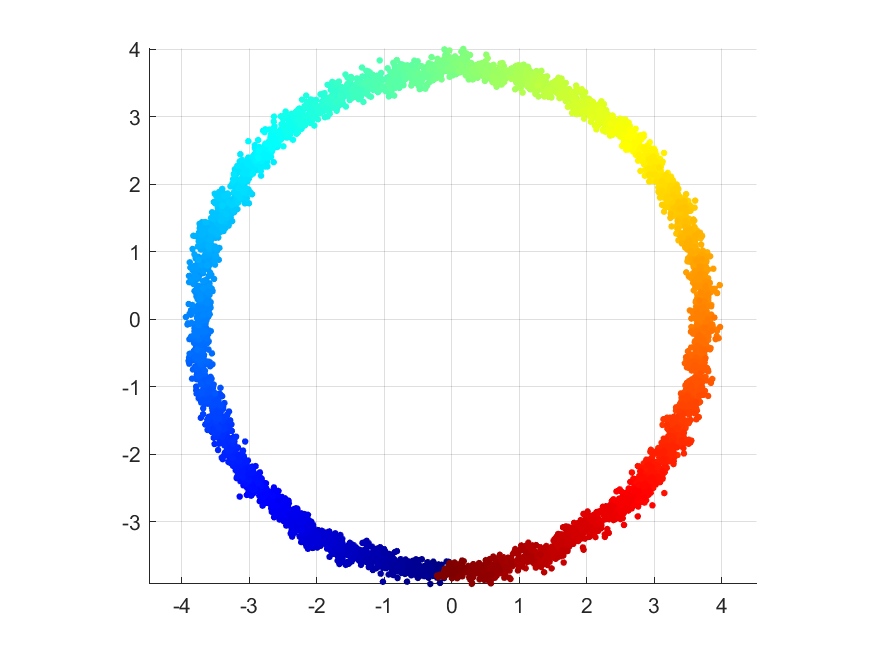}}
  \subfigure[250 iterations]
  {\includegraphics[width=5cm]{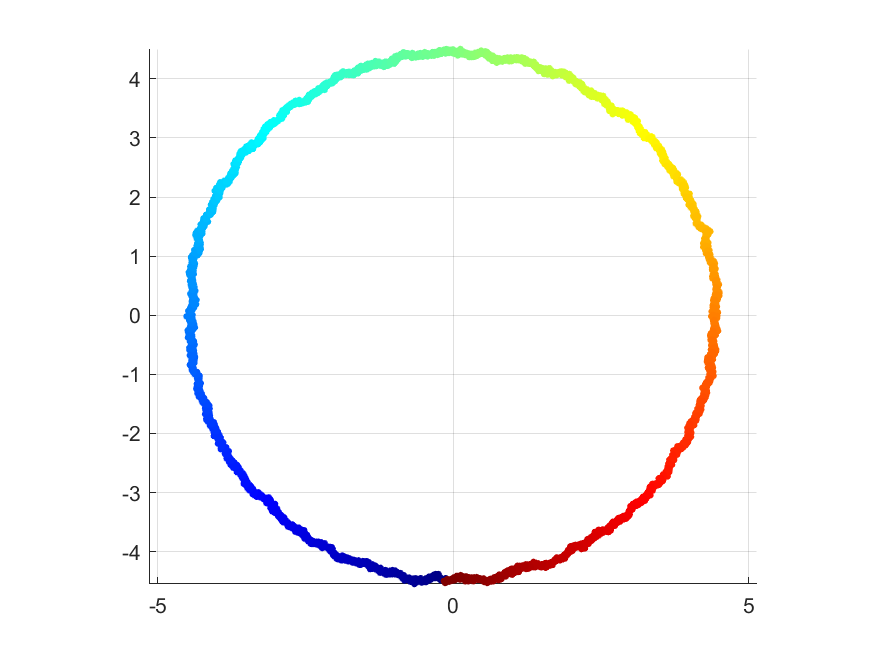}}
}
\caption{Computation of Helix data using sDRFC}\label{fig:helix}
\end{figure}

To assess the quality of the resulting low-dimensional representations, we examine the preserved local structure of the data~\cite{vandermaaten2009}. We measure this through the generalization errors of 1-nearest neighbor classifiers trained on the low-dimensional data representation~\cite{Guido08}, as well as evaluating the \textit{trustworthiness} and \textit{continuity} of the low-dimensional embeddings~\cite{VennaKaski06}. The trustworthiness measures the proportion of points that are positioned too closely in the low-dimensional space. The trustworthiness is defined as
\begin{equation*}
 T(k) = 1-\frac{2}{nk(2n-3k-1)}\sum_{i=1}^n\sum_{j\in U_i^{(k)}}\big(r(i, j)-k\big),
\end{equation*}
where $r(i, j)$ represents the rank of the point $j$ based on pairwise distances in the low-dimensional representations. The set $U_i^{(k)}$ consists of the points that are among the $k$ nearest neighbors in the low-dimensional space but not in the high-dimensional space. The continuity
measure is defined as
\begin{equation*}
C(k) = 1-\frac{2}{nk(2n-3k-1)}\sum_{i=1}^n\sum_{j\in V_i^{(k)}}\big(\hat{r}(i, j)-k\big),
\end{equation*}
where $\hat{r}(i, j)$ represents the rank of the point $j$ based on the pairwise distances between the points in the high-dimensional space. The set $V_i^{(k)}$ contains the points that are among the $k$ nearest neighbors in the high-dimensional space but not in the low-dimensional space.

For comparison with other methods, we import the computational results from~\cite{vandermaaten2009}. The compared methods are convex techniques, including Principal Components Analysis (PCA), Isomap, Kernel PCA (KPCA),  Maximum Variance Unfolding (MVU), diffusion maps (DM), Local Linear Embedding (LLE),  Laplacian Eigenmaps (LEM), Hessian LLE (HLLE), and Local Tangent Space Analysis (LTSA). The first five methods are classified as full spectral techniques, and the last four are sparse spectral techniques. Although~\cite{vandermaaten2009} mentions additional non-convex techniques, we omit them here as there are no significant differences in comparison. For detailed descriptions and references of these methods, we refer readers to~\cite{vandermaaten2009}.

The generalization errors of 1-nearest neighbor classifiers trained on the low-dimensional data representation are given in Table~\ref{table:gen_error}. In the table, the leftmost column include the abbreviation for the dataset and the target dimensionality to transform the high dimensional data. ``None" refers to the results obtained in the original $\mbR^D$ space. The best performing method is highlighted in bold  in each dataset. To evaluate the proposed methods, we report the average errors over $5$ trials. The performance of the proposed methods is the best among them for Swiss roll, Broken Swiss roll, and COIL20 and highly ranked for other cases.

The trustworthiness and continuity calculated using the $12$ nearest neighbors are reported in Tables \ref{table:truth} and \ref{table:conti}, respectively. The proposed method shows the best or top-ranked performance, conceivably due to the local distance constraint \eqref{model:local} or the first term in the system \eqref{model:full}.

\begin{table}[t]
 \center{
	\caption{ Generalization errors (\%) of 1-NN classifiers}\label{table:gen_error}
  {\tiny\addtolength{\tabcolsep}{-3.5pt}
\begin{tabular}{|c|c||c|c|c|c|c||c|c|c|c||c|c|} 
  \hline			
  Dataset (d) & None & PCA & Isomap & KPCA & MVU & DM & LLE & LEM & HLLE & LTSA & kDRFC & sDRFC \\
  \hline
Swiss roll (2D) & 3.68 & 29.76 & 3.40 & 30.24 & 4.12 & 33.50 & 3.74 & 22.06 & 3.56 & 3.90 & \textbf{2.77}  &  2.86 \\
Helix (1D) & 1.24 & 35.50 & 13.18 & 38.04 & 7.48 & 35.44 & 32.32 & 15.24 & 52.22 & \textbf{0.92} &  2.93 &  3.06 \\
Twin peaks (2D) & 0.40 & 0.26 & 0.22 & \textbf{0.12} & 0.56 & 0.26 & 0.94 & 0.88 & 0.14 & 0.18 & 0.42 & 0.55 \\
Broken Swiss (2D) & 2.14 & 25.96 & 14.48 & 32.06 & 32.06 & 58.26 & 36.94 & 10.66 & 6.48 & 15.86 & \textbf{3.60}  & 3.81 \\
\hline
MNIST (20D) & 5.11 &  \textbf{6.74} & 12.64 & 13.86 & 13.58 & 25.00 & 10.02 & 11.30 & 91.66 & 90.32 & 10.35  & 8.80  \\
COIL20 (5D) & 0.14 & 3.82 & 15.69 & 7.78 & 25.14 & 11.18 & 22.29 & 95.00 & 50.35 & 4.17 & 5.57  &  \textbf{3.13} \\
ORL (8D) & 2.50 & \textbf{4.75} & 27.50 & 6.25 & 24.25 & 90.00 & 11.00 & 97.50 & 56.00 & 12.75 &  12.65 & 8.25 \\
HIVA (15D) & 4.63 & 5.05 & 4.92 & 5.07 & 4.94 & 5.46 & 4.97 & 4.81 & \textbf{3.51} & \textbf{3.51} & 4.96 & 5.00 \\
  \hline
\end{tabular}
  }
 }
\end{table}

\begin{table}[t]
  \center{
   \caption{Trustworthinesses T(12)}\label{table:truth}
   {\tiny\addtolength{\tabcolsep}{-2pt}
 \begin{tabular}{|c|c|c|c|c|c||c|c|c|c||c|c|} 
   \hline			
   Dataset (d) & PCA & Isomap & KPCA & MVU & DM & LLE & LEM & HLLE & LTSA & kDRFC & sDRFC \\
   \hline
 Swiss roll (2D) &  0.88 & 0.99 & 0.88 & 1.00 & 0.81 & \textbf{1.00} & 0.92 & \textbf{1.00} & \textbf{1.00}  & \textbf{1.00} &  \textbf{1.00} \\
 Helix (1D) &  0.78 & 0.74 & 0.71 & 0.96 & 0.73 & 0.83 & 0.87 & 0.35 & \textbf{1.00} & \textbf{1.00} & \textbf{1.00} \\
 Twin peaks (2D) &  0.98 & 0.98 & 0.99 & 0.99 & \textbf{1.00} & 0.99 & 0.99 & 0.99 & 0.99 & 0.99 & 0.99\\
 Broken Swiss (2D)  & 0.96 & 0.97 & 0.96 & 0.97 & 0.78 & 0.94 & 0.97 & 0.92 & 0.89 & \textbf{1.00} & \textbf{1.00} \\
 \hline
 MNIST (20D) &  \textbf{1.00} & 0.96 & 0.99 & 0.92 & 0.95 & 0.96 & 0.89 & 0.54 & 0.54 & 0.99 &  0.99   \\
 COIL20 (5D) &  0.99 & 0.89 & 0.98 & 0.92 & 0.91 & 0.93 & 0.27 & 0.69 & 0.96 & 0.99 &  \textbf{1.00}    \\
 ORL (8D) &  \textbf{0.99} & 0.78 & 0.98 & 0.95 & 0.49 & 0.95 & 0.29 & 0.76 & 0.94 & 0.98 & \textbf{0.99}  \\
 HIVA (15D) &  0.97 & 0.87 & 0.89 & 0.89 & 0.75 & 0.80 & 0.78 & 0.42 & 0.54 & 0.97 &  \textbf{0.98} \\
   \hline
 \end{tabular}
   }
  }
 \end{table}

 \begin{table}[t]
  \center{
   \caption{Continuity C(12) on the synthetic datasets}\label{table:conti}
   {\tiny\addtolength{\tabcolsep}{-2pt}
 \begin{tabular}{|c|c|c|c|c|c||c|c|c|c||c|c|} 
   \hline			
   Dataset (d) & PCA & Isomap & KPCA & MVU & DM & LLE & LEM & HLLE & LTSA & kDRFC & sDRFC\\
   \hline		
   Swiss roll (2D)  & \textbf{1.00} & 0.99 & 0.99 & \textbf{1.00} & 0.91 & 1.00 & 0.99 & \textbf{1.00} & \textbf{1.00}  & \textbf{1.00}  &  \textbf{1.00}  \\
   Helix (1D) &   0.98 & 0.97 & 0.98 & \textbf{1.00} & 0.98 & 0.99 & 0.99 & 0.50 & \textbf{1.00}  &  \textbf{1.00} &  \textbf{1.00}  \\
   Twin peaks (2D)  & \textbf{1.00} & 0.99 & 0.99 & \textbf{1.00} & \textbf{1.00} & 0.99 & \textbf{1.00} & \textbf{1.00} & \textbf{1.00}  & \textbf{1.00} & \textbf{1.00} \\
   Broken Swiss (2D) & \textbf{1.00} & 0.98 & 0.99 & \textbf{1.00} & 0.90 & 0.98 & 0.99 & 0.99 & 0.99 & \textbf{1.00} & \textbf{1.00} \\
   \hline		
   MNIST (20D) & \textbf{1.00} & 0.94 & 0.89 & 0.93 & 0.95 & 0.96 & 0.70 & 0.50 & 0.50 & 0.99 & \textbf{1.00}  \\
   COIL20 (5D) & \textbf{1.00} & 0.90 & 0.98 & 0.97 & 0.92 & 0.95 & 0.47 & 0.71 & 0.99 & \textbf{1.00}  & \textbf{1.00}   \\
   ORL (8D) &  \textbf{0.99} & 0.76 & 0.95 & 0.97 & 0.57 & 0.95 & 0.49 & 0.76 & 0.94 & \textbf{0.99} & \textbf{0.99} \\
   HIVA (15D) &  \textbf{0.99} & 0.84 & 0.88 & 0.94 & 0.80 & 0.80 & 0.54 & 0.51 & 0.62 & \textbf{0.99} & \textbf{0.99} \\
   \hline
 \end{tabular}
   }
  }
 \end{table}

\section{Conclusion}
In this paper, we propose a new dimensionality reduction model inspired by the formations of mobile agents under interagent distance control. We regard reduction process as a dynamical system interacting between many bodies, aiming to achieve local and global structures from the control of neighbor points and that of remote points. Numerical tests and comparisons validate our approach. Our next challenge includes more rigorous mathematical analysis and robust computational schemes applicable to massive data sets.

\section*{Acknowledgement}
This work was supported by by the National Research Foundation of Korea (NRF) (No. 2022R1A2C1010537).

\bibliographystyle{plain}

\begin{thebibliography}{10}

  \bibitem{Ahn20}
  Hyo-Sung Ahn.
  \newblock {\em Formation control}, volume 205 of {\em Studies in Systems, Decision and Control}.
  \newblock Springer, 2020.

  \bibitem{Cunningham15}
  John~P. Cunningham and Zoubin Ghahramani.
  \newblock Linear dimensionality reduction: Survey, insights, and generalizations.
  \newblock {\em Journal of Machine Learning Research}, 16(89):2859--2900, 2015.

  \bibitem{DimarogonasJohansson10}
  Dimos~V. Dimarogonas and Karl~H. Johansson.
  \newblock Stability analysis for multi-agent systems using the incidence matrix: quantized communication and formation control.
  \newblock {\em Automatica J. IFAC}, 46(4):695--700, 2010.

  \bibitem{Ghodsi2006}
  Ali Ghodsi.
  \newblock Dimensionality reduction a short tutorial.
  \newblock 2006.

  \bibitem{HaLiXue13}
  Seung-Yeal Ha, Zhuchun Li, and Xiaoping Xue.
  \newblock Formation of phase-locked states in a population of locally interacting {K}uramoto oscillators.
  \newblock {\em J. Differential Equations}, 255(10):3053--3070, 2013.

  \bibitem{Hartman02}
  Philip Hartman.
  \newblock {\em Ordinary Differential Equations}.
  \newblock Society for Industrial and Applied Mathematics, second edition, 2002.

  \bibitem{KrickBrouckeFrancis09}
  Laura Krick, Mireille~E. Broucke, and Bruce~A. Francis.
  \newblock Stabilisation of infinitesimally rigid formations of multi-robot networks.
  \newblock {\em Internat. J. Control}, 82(3):423--439, 2009.

  \bibitem{OhAhn14}
  Kwang-Kyo Oh and Hyo-Sung Ahn.
  \newblock Distance-based undirected formations of single-integrator and double-integrator modeled agents in {$n$}-dimensional space.
  \newblock {\em Internat. J. Robust Nonlinear Control}, 24(12):1809--1820, 2014.

  \bibitem{OhParkAhn15}
  Kwang-Kyo Oh, Myoung-Chul Park, and Hyo-Sung Ahn.
  \newblock A survey of multi-agent formation control.
  \newblock {\em Automatica J. IFAC}, 53:424--440, 2015.

  \bibitem{Olfati-Saber06}
  Reza Olfati-Saber.
  \newblock Flocking for multi-agent dynamic systems: algorithms and theory.
  \newblock {\em IEEE Transactions on Automatic Control}, 51(3):401--420, 2006.

  \bibitem{Guido08}
  Guido Sanguinetti.
  \newblock Dimensionality reduction of clustered data sets.
  \newblock {\em IEEE Transactions on Pattern Analysis and Machine Intelligence}, 30(3):535--540, 2008.

  \bibitem{SaulWeinberger06}
  Lawrence~K. Saul, Kilian~Q. Weinberger, Fei Sha, Jihun Ham, and Daniel~D. Lee.
  \newblock Spectral methods for dimensionality reduction.
  \newblock In Olivier Chapelle, Bernhard Scholkopf, and Alexander Zien, editors, {\em Semi-supervised learning}, pages 293--308. MIT Press, Cambridge, MA, 2006.

  \bibitem{Sorzano2014}
  Carlos Sorzano, Javier Vargas, and A.~Montano.
  \newblock A survey of dimensionality reduction techniques.
  \newblock 03 2014.

  \bibitem{Sun18}
  Zhiyong Sun.
  \newblock {\em Cooperative coordination and formation control for multi-agent systems}.
  \newblock Springer Theses. Springer, Cham, 2018.
  \newblock Doctoral Thesis accepted by The Australian National University, Canberra, Australia.

  \bibitem{vandermaaten2009}
  Laurens Van Der~Maaten, Eric Postma, and Jaap Van~den Herik.
  \newblock Dimensionality reduction: a comparative review.
  \newblock {\em Tilburg University, TiCC-TR 2009-005}, pages 1--35, 2008.

  \bibitem{VennaKaski06}
  Jarkko Venna and Samuel Kaski.
  \newblock Visualizing gene interaction graphs with local multidimensional scaling.
  \newblock In {\em Proceedings of the 14th European Symposium on Artificial Neural Networks}, pages 557--562, 2006.

  \end{thebibliography}

\end{document}